\documentclass{ifacconf}

\usepackage{graphicx} % include this line if your document contains figures
\usepackage{natbib} % required for bibliography

\usepackage{amsmath} % assumes amsmath package installed
\usepackage{amssymb} % assumes amsmath package installed
\usepackage{acronym} % used for the automatic introduction of acronyms
\usepackage{algorithm} % used for the definition of algorithms
\usepackage{algpseudocode} % used for the definition of algorithms
\usepackage[short,c2]{optidef} % for formulating optimization problems
\usepackage{xcolor} % required for \color{} definitions
\usepackage{booktabs} % required for \toprule, \midrule and \bottomrule
\usepackage{balance} % to balance columns on last page
\usepackage{tikz} % required for the pgf plots

\graphicspath{{figures/}}
\newcommand{\executeiffilenewer}[3]{%
	\ifnum\pdfstrcmp{\pdffilemoddate{#1}}%
	{\pdffilemoddate{#2}}>0%
	{\immediate\write18{#3}}\fi%
}

\newcommand{\priority}[2][1]{%
	\ifnum#1>1% the number at the end of this line is the lowest priority that will still show up normally
	\ifnum#1>10% lowest priority that is still printed in gray
	\else%
	\textcolor{gray}{#2}% comment out this line to remove priority-2 text entirely.
	\fi%
	\else%
	#2%
	\fi%
}

\newcommand{\set}[1]{\mathcal{ #1 }} % a set

\newcommand{\ub}[1]{\overline{ #1 }} % lower bound
\newcommand{\R}{\mathbb{R}} % the real numbers
\newcommand{\Id}{\mathbb{I}} % the identity matrix
\newcommand{\T}{^\top} % the transpose operator
\newcommand{\norm}[1]{\left\lVert #1 \right\rVert} % the norm
\newcommand{\normalized}[2][]{\frac{#2}{\norm{#2}_{#1}}} % a normalized vector
\newcommand{\dpartial}[2]{\frac{\partial {#1}}{\partial {#2}}} % partial derivative
\newcommand{\state}{\vec{x}}% the state vector
\newcommand{\controls}{\vec{u}}% the control vector
\newcommand{\obs}{\vec{o}}% a point on an obstacle / in the occupied set
\newcommand{\pos}{\vec{p}}% position of the robot
\newcommand{\cen}{\vec{c}}% center of a free region
\newcommand{\nidx}{\star}% norm index

\newcommand{\robot}{\set{R}}% the set of points occupied by the robot
\newcommand{\robotHull}{\ub{\robot}}% a convex bounding polytope that containts \robot
\newcommand{\free}{\set{F}^\nidx}% the free set
\newcommand{\freeRegion}[1][\cen]{\set{C}^\nidx_{#1}}% a free region (centered around \cen)
\newcommand{\occupied}{\set{O}}% the occuped set

\newcommand{\inistate}{\state_\mathrm{s}}% the initial / current state
\newcommand{\terminalstate}{\state_\mathrm{g}}% the goal state
\newcommand{\nstate}{{n_{\mathrm{\state}}}}% number of states n_\state
\newcommand{\ncontrols}{{n_{\mathrm{\controls}}}}% number of controls n_\controls
\newcommand{\dist}[1][\occupied]{d_{#1}^\nidx}% the distance function
\newcommand{\sd}[1][\occupied]{\mathrm{sd}_{#1}^\nidx}% the signed distance function
\newcommand{\penetration}[1][\occupied]{\mathrm{pen}_{#1}^\nidx}% the penetration measure
\newcommand{\D}{\mathrm{D}}% subscript for discrete time quantities

\newcommand{\interior}{\mathrm{int}}% interior of a set

\newtheorem{remark}{Remark}

\newtheorem{lemma}{Lemma}
\newtheorem{assumption}{Assumption}

\newenvironment{proof}
{\textit{Proof:} }
{\hfill $\square$}

\algdef{SE}[DOWHILE]{Do}{DoWhile}{\algorithmicdo}[1]{\algorithmicwhile\ #1}%

\begin{document}
\begin{frontmatter}

\title{CIAO$^\star$: MPC-based Safe Motion Planning in Predictable Dynamic Environments}

\thanks[footnoteinfo]{pronounced \emph{`ciao-star`}, where $\star$ is a wildcard that specifies the norm used for the \acl*{fr}.\\
	This research was supported by the German Federal Ministry for Economic Affairs and Energy (BMWi) via eco4wind (0324125B) and DyConPV (0324166B), by DFG via Research Unit FOR 2401, and the EU’s Horizon
	2020 research and innovation program under grant agreement No 732737
	(ILIAD).}

\author[First]{Tobias Schoels}
\author[First]{Per Rutquist}
\author[Second]{Luigi Palmieri}
\author[First]{Andrea Zanelli}
\author[Second]{Kai O. Arras}
\author[First,Third]{Moritz Diehl}

\address[First]{Department of Microsystems Engineering, University of Freiburg
	(e-mail: \{tobias.schoels, per.rutquist, andrea.zanelli, moritz.diehl\}@imtek.uni-freiburg.de)}
\address[Second]{Robert Bosch GmbH, Corporate Research, Stuttgart, Germany
	(e-mail: \{luigi.palmieri, kaioliver.arras\}@de.bosch.com)}
\address[Third]{Department of Mathematics, University of Freiburg}
%\address[Third]{Electrical Engineering Department,
% Seoul National University, Seoul, Korea, (e-mail: author@snu.ac.kr)}

\begin{abstract} % Abstract of not more than 250 words.
% Motivation
Robots have been operating in dynamic environments and shared workspaces for decades.
% Existing work
Most optimization based motion planning methods, however, do not consider the movement of other agents, e.g. humans or other robots, and therefore do not guarantee collision avoidance in such scenarios.
% Novelty
This paper builds upon the \ac*{ciao} method and proposes a motion planning algorithm that guarantees collision avoidance in predictable dynamic environments.
Furthermore, it generalizes CIAO's free region concept to arbitrary norms and proposes a cost function to approximate time optimal motion planning.
% Detail Method briefly
The proposed method, CIAO$^\star$, finds kinodynamically feasible and collision free trajectories for constrained single body robots using \ac*{mpc}.
% Based on the predicted movement of surrounding agents CIAO$^\star$ optimizes the motion of a single body agent.
It optimizes the motion of one agent and accounts for the predicted movement of surrounding agents and obstacles.
% Results
The experimental evaluation shows that CIAO$^\nidx$ reaches close to time optimal behavior.
\end{abstract}


\end{frontmatter}
%===============================================================================

\section{Introduction}
Safe and smooth robot navigation is still an open challenge particularly for autonomous systems navigating in shared spaces with humans (e.g. intra--logistic and service robotics) and in densely crowded environments \citep{triebel2016spencer}. In these scenarios, the reactive avoidance of dynamic obstacles is an important requirement. Combined with the objective of reaching time optimal robot behavior, this poses a major challenge for motion planning and control and remains subject of active research.

%
%Many applications in mobile robotics (e.g. intra--logistic and service robotics) require robots to operate in crowded dynamic environments among other agents \cite{triebel2016spencer}, such as humans or other autonomous systems. In these scenarios, the reactive avoidance of dynamic obstacles is an important requirement. Combined with the objective of reaching time optimal robot behavior, this poses a major challenge for motion planning and control and remains subject of active research.
% cover figure
\begin{figure}
	\centering
	\includegraphics[width=\columnwidth]{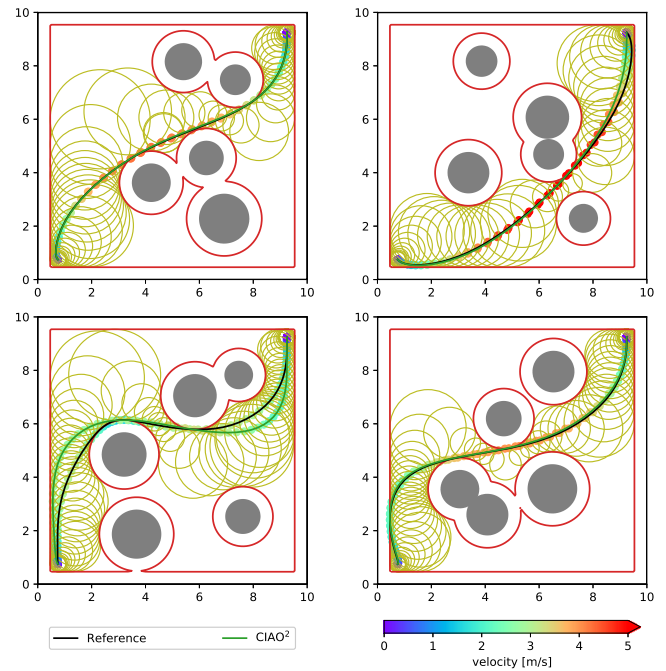}
	\caption{Example trajectories found by CIAO$^2$ in green and a time optimal reference in black. The olive colored circles mark the \acfp*{fr}, the red lines the safety margin to obstacles $\rho$.}
	\label{fig:to_traj}
\end{figure}

Recent approaches tackle collision avoidance by formulating and solving optimization problems \citep{Schulman2014, Bonalli2019, schoels2019nmpc}. These approaches offer good performance for finding locally optimal solutions but offer no guarantee to find the global optimum.
Sampling-based planners, cf. \cite{Karaman2011}, on the other hand, are asymptotically optimal and have been extended to dynamic environments \citep{Otte2015}.
They are, however, slow to converge and are therefore commonly terminated early.
The suboptimal result is then passed to a trajectory optimization algorithm, like CIAO$^\star$, the one proposed in this paper.

\subsubsection{Related Work:} \label{sec:related_work}
%Several approaches have been proposed recently to solve the motion planning problem with numerical optimization techniques.
%
A shortcoming of most common trajectory optimization methods is their incapability to respect dynamic obstacles and kinodynamic constraints, e.g. bounds on the
acceleration, and a lack of timing in their predictions \citep{Quinlan1993, Zucker2013, Schulman2014}.
These approaches are typically limited to the optimization of paths rather than trajectories and impose constraints by introducing penalties.
%The increase of computing power and the availability of fast numerical solvers, as for example discussed by \cite{Kouzoupis2018}, has given rise to \ac{mpc} based approaches.

Classical approaches to obstacle avoidance include \cite{Borenstein1991, Fox1997, Ko1998, Fiorini1998, Minguez2004,Quinlan1993}. In contrast to our approach, they do neither produce optimal trajectories, nor account for the robot's dynamics and constraints, nor handle the obstacles with their full shape (i.e. with convex hulls) and predicted future movements.

Popular recent trajectory optimization methods include\linebreak CHOMP \citep{Zucker2013}, TrajOpt \citep{Schulman2014}, and GuSTO \citep{Bonalli2019}, which find smooth trajectories in static environments efficiently.
%% MPC based methods
An increasing number of these approaches use \ac{mpc} based formulations to obtain kinodynamically feasible trajectories, e.g. \citep{Bonalli2019, Herbert2017, Zhang2017}.
In this framework an \ac{ocp} is solved in every iteration.
%% SCP
The \acp{ocp} formulated by CIAO$^\nidx$ are convex, which makes it a \ac{scp} method like TrajOpt and GuSTO.
%% SLP
In some cases the formulated problems are linear, such that we obtain a \ac{slp} method. %, as first proposed by \cite{Griffith1961}.

%% signed-distance function
Similarly to TrajOpt \citep{Schulman2014} and GuSTO \citep{Bonalli2019}, CIAO$^\nidx$ uses a \ac{sdf} to model the environment. While they linearize the \ac{sdf}, we find a convex-inner approximation (as depicted in Fig.~\ref{fig:to_traj}) and propose a continuous time collision avoidance constraint, instead of a penalty term in the cost function.
%% P-norm
CIAO$^\nidx$ generalizes the \ac{sdf} (and thereby also the collision avoidance constraint) to arbitrary norms, similarly to OBCA \citep{Zhang2017} and in \cite{hyun2017new}.
%% Time optimality
Moreover, as \cite{Roesmann2017}, it approximates time optimal behavior.

%% MPC based collision avoidance & trajectory tracking
\ac{mpc} has been used to combine trajectory tracking and collision avoidance, e.g. \cite{Lim2008}. CIAO \citep{schoels2019nmpc} goes one step further and also allows for trajectory optimization.
%% CIAO
Additionally it preserves feasibility across iterations using a convex collision avoidance constraint that is based on the Euclidean distance to the closest obstacle.
It has been shown to work well in dynamic environments, even for robots with nonlinear dynamics.
This paper presents a generalization of CIAO.

\subsubsection{Contribution:}
%% Obstacles prediction
In contrast to all other methods listed above, CIAO$^\nidx$ guarantees collision avoidance in predictable dynamic environments.
Further it differs from the original CIAO \citep{schoels2019nmpc} in four regards:
\begin{itemize}
	\item CIAO$^\nidx$ is norm agnostic, such that the original CIAO is a special case where $\nidx=2$, i.e. CIAO $\equiv$ CIAO\textsuperscript{2}.
	\item The (predicted) movement of dynamic obstacles is considered explicitly during trajectory optimization.
	\item We motivate the use of a different cost function to approximate time optimal behavior.
	\item The collision avoidance constraint is generalized to robots of shapes that can be approximated by a convex, bounding polytope.
	%\item We propose a generic way to grow \acfp*{fr}.
\end{itemize}
To the best of the author's knowledge this makes CIAO$^\nidx$ the first MPC approach that approximates time optimal behavior in predictable dynamic environments and guarantees collision avoidance for linear systems.
Like the original CIAO we present formulations for both offline trajectory optimization and online \ac{mpc} based obstacle avoidance and control.
In coherence with the theory, the experiments in this paper consider a linear system. CIAO$^\star$'s applicability to constrained, nonlinear systems has been demonstrated by \cite{schoels2019nmpc}.

\subsubsection{Structure:}
Sec.~\ref{sec:problem} formalizes the trajectory optimization problem in dynamic environments we want to solve. The proposed approach, CIAO$^\nidx$, is introduced in Sec.~\ref{sec:approach} including some theoretical considerations. Sec.~\ref{sec:ciao_for_motion_planning} details two algorithms that use CIAO$^\nidx$ for motion planning. The experiments and results are discussed in Sec.~\ref{sec:experiments}. A summary and an outlook is given in Sec.~\ref{sec:conclusion}. In App.~\ref{sec:taylor_ub} we introduce the \emph{Taylor upper bound}, an approach to continuous time constraint satisfaction.

%Current obstacle avoidance methods are immature for industrialization for lack of guarantees, optimality, generalization capabilities with respect to environments/vehicle properties.
%Explainability is important for technology transfer into products.

\section{Time Optimal Motion Planning} \label{sec:problem}

We formalize time optimal motion planning as a continuous time \acf{ocp}:
\begin{mini!}
	{\state(\cdot), \controls(\cdot), T}% optimization variable
	{T} % objective Function and label for the objective
	{\label{eq:ocp}} % label of the optimization problem
	{} % result of optimization, e.g. J^* =
	%\addConstraint{LHS.1}{RHS.1\label{Const1}}{extraConst1}
	\addConstraint{0}{\leq T}
	\addConstraint{\state(0)}{= \inistate}{}
	\addConstraint{\state(T)}{= \terminalstate}{}
	\addConstraint{\dot{\state}(t)}{= A \state(t) + B \controls(t),}{\quad t \in [0, T]}
	\addConstraint{(\state(t), \controls(t))}{ \in \set{H} ,}{\quad t \in [0, T]}
%	\addConstraint{0}{\geq h(\state(t), \controls(t)) ,}{\quad t \in [0, T]}
	\addConstraint{\emptyset}{= \interior(\robot(\state(t))) \cap \occupied(t), \label{eq:ocp_cac}}{\quad t \in [0, T],}
\end{mini!}
where $\state(\cdot): \R \rightarrow \R^{\nstate}$ denotes the robot's state, $\controls(\cdot): \R \rightarrow \R^{\ncontrols}$ is the vector of controls, $T$ is the length of the trajectory in seconds, which is minimized. The fixed vector $\inistate$ is the robot's current state, and $\terminalstate$ is the goal state.
We use the common shorthand $\dot{\state}$ to denote the derivative with respect to time, i.e. $\dot{\state} = \frac{\mathrm{d}\, \state}{\mathrm{d} t}$. The expression $A\state + B\controls$ denotes the system's linear dynamical model, the convex polytopic set $\set{H}$ implements path constraints, e.g. physical limitations of the system. The open set $\interior(\robot(\state(t))) \subset \R^n$ is the interior of the set of points occupied by the robot at time $t$ and $\occupied(t) \subset \R^n$ is the set of points occupied by obstacles at that time. Finally, $n$ is the dimension of the robot's workspace $\R^n$.
Note that for fixed $T$, problem \eqref{eq:ocp} becomes convex if \eqref{eq:ocp_cac} is removed.

\section{CIAO$^\star$: Convex Inner ApprOximation} \label{sec:approach}
In the following we detail the reformulations and parameterizations of \eqref{eq:ocp} used in this paper.
First, we introduce our implementation of the collision avoidance constraint. It is based on the concept of \acfp{fr} that utilizes the \acf{sdf}. Then we discretize \eqref{eq:ocp} followed by a discussion of safety in continuous time.

\subsection{Collision Avoidance Constraint} \label{sec:collision_avoidance}

The collision avoidance constraint \eqref{eq:ocp_cac} can be formulated using the \acf{sdf} $\sd(\cdot)$ for an arbitrary occupied set $\occupied$. The distance of a given point $\pos \in \R^n$ to the occupied set $\occupied$ is defined as
\begin{equation} \label{eq:d}
	\dist(\pos) = \min_{\mathbf{o} \in \occupied} \norm{\pos - \mathbf{o}}_\nidx.
\end{equation}
Note that $\dist(\pos) = 0$ for $\pos \in \occupied$ and that $\nidx$ is used as a wildcard, not a dual norm notation.
Further, we define the penetration depth as the distance of $\pos \in \R^n$ to the unoccupied set $\R^n \setminus \occupied$
\begin{equation} \label{eq:pen}
\penetration(\pos) = \min_{\mathbf{o} \in \R^n \setminus \occupied} \norm{\pos - \mathbf{o}}_\nidx .
\end{equation}
Note that $\penetration(\pos) = 0 $ for $ \pos \notin \occupied$.
We combine \eqref{eq:d} and \eqref{eq:pen} to obtain the \ac{sdf}
\begin{equation} \label{eq:dist}
	\sd(\pos) = \dist(\pos) - \penetration(\pos).
\end{equation}
Note that the \ac{sdf} is in general non-linear, non-convex, and non-differentiable, but continuous.
An illustration of the \ac{sdf} for $\nidx = \{1, 2, \infty \}$ is shown in Fig.~\ref{fig:sdf_example}.
\begin{figure}
	\centering
	\includegraphics[width=\columnwidth]{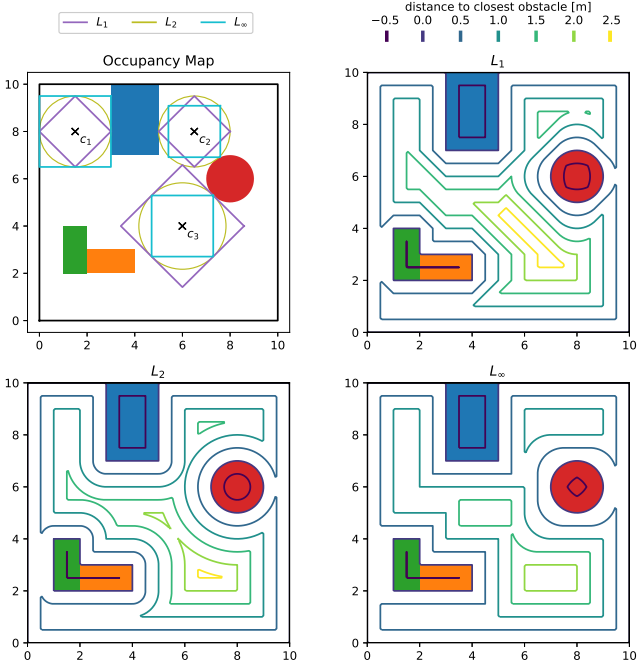}
	\caption{Contour lines of the s	igned distance function for $L_1, L_2$, and $L_\infty$ norm in an example environment.
		Top left shows the resulting \aclp*{fr} for hand picked locations.}
	\label{fig:sdf_example}
\end{figure}

The \emph{free set} $\free$ contains all points that are unoccupied:
\begin{equation} \label{eq:free}
\free = \{ \cen \in \R^n : \sd(\cen) \ge 0 \}.
\end{equation}
The full-body collision avoidance can now be formulated via the \ac{sdf} as $\interior(\robot) \subseteq \free$.
\begin{lemma} \label{lem:sd_for_cac}
	Let $\occupied$ and $\robot$ be the set of points occupied by obstacles and the robot respectively, then
	\begin{equation*}
		\occupied \cap \interior(\robot) = \emptyset \Leftrightarrow \interior(\robot) \subseteq \free \Leftrightarrow \sd(\pos) \ge 0, \; \forall\, \pos \in \robot.
	\end{equation*}
\end{lemma}%
\begin{proof}
	This follows directly from the definitions of the distance function \eqref{eq:dist} and the free set \eqref{eq:free}.
	%	\todo{Alternative: show by contradiction, using triangle inequality.}
\end{proof}%

\subsubsection{Convex Free Regions (\acp{fr}):}

The collision avoidance constraint is in general non-convex, non-linear, and non-differentiable, cf. \cite{Schulman2014}. This poses a problem for derivative based optimization methods, that are used to solve the \ac{ocp} in \eqref{eq:ocp}. To overcome this problem we use a convex inner approximation of this constraint, called \acf{fr}.

For any point $\cen \in \free$ the \acl{sdf} $\sd(\cen)$ yields a \emph{\acl{fr}} $\freeRegion$, which we define as
\begin{equation} \label{eq:freeRegion}
\freeRegion = \{\pos \in \R^n : \norm{\pos - \cen}_\nidx \leq \sd(\cen) \}.
\end{equation}
Note that \ac{fr} $\freeRegion$ is fully described by its center point $\cen$ and the used norm $\nidx$.
Fig.~\ref{fig:sdf_example} shows \acp{fr} obtained for different, hand-picked points in an example environment.
Now we show that \acp{fr} are convex subsets of the free set.
\begin{lemma} \label{lem:convex_inner_approximation}
	For any free point $\cen \in \free$ the \acl{fr} $\freeRegion$ is a convex subset of $\free$, i.e. $\freeRegion \subseteq \free$.
\end{lemma}
\begin{proof}
	We prove this lemma in two steps. First, we observe that \aclp{fr} are norm balls and therefore convex. Second, we show that $\cen \in \free \Rightarrow \freeRegion \in \free$ by construction.\\
	% inner approximation a la Moritz
	Take any $\pos \in \freeRegion$ and any $\obs \in \occupied$, then the reverse triangle inequality yields $\norm{\pos - \obs}_\nidx \geq \norm{\obs - \cen}_\nidx - \norm{\pos - \cen}_\nidx $. Since $\norm{\cen - \obs}_\nidx \geq \sd(\cen)$ due to \eqref{eq:dist} and $-\norm{\pos - \cen}_\nidx \geq -\sd(c)$ due to \eqref{eq:freeRegion}, we get $\norm{\pos - \obs}_\nidx \geq \sd(\cen) - \sd(\cen) = 0$.
%	$\norm{\obs - \cen}_\nidx - \norm{\pos - \cen}_\nidx > 0, \forall \pos \in \freeRegion$, since $-\norm{\pos - \cen}_\nidx \geq -\sd(c)$ and $\cen \in \free$ implies $\norm{\cen - \obs}_\nidx \geq \sd(\cen) > 0, \forall \obs \in \occupied$.
%	Thereby we get $\norm{\pos - \obs}_\nidx \geq \norm{\obs - \cen}_\nidx - \sd(\cen) > 0$.
\end{proof}
\begin{remark}
	This lemma is a generalization of Lem.~2 by \cite{schoels2019nmpc} to arbitrary norms $\norm{\cdot}_\nidx$.
\end{remark}

\subsubsection{Full-body collision avoidance:}

To implement full-body collision avoidance efficiently, we approximate the robot's shape by a convex polytope, like \cite{Schulman2014}.
\begin{assumption} \label{ass:robot_hull}
	Assume that a finite set of points $\robotHull = \{ \mathbf{\nu}_1, \ldots, \mathbf{\nu}_{\mathrm{n_\robot}} \}$ exists, such that $\robot \subseteq \mathrm{convhull}(\robotHull)$.
\end{assumption}
\begin{remark}
Collision avoidance can be enforced by constraining the spanning vertices to a free region, i.e. $\mathrm{convhull}(\robotHull) \subseteq \freeRegion \Leftrightarrow \nu_1, \ldots, \nu_{\mathrm{n_\robot}} \in \freeRegion$.
It is easy to show that this is an inner approximation of the actual constraint, i.e. $\interior(\bar{\robot}) \subseteq \freeRegion \Rightarrow \interior(\robot) \subseteq \freeRegion \Rightarrow \interior(\robot) \subseteq \free$.
\end{remark}

\subsection{Enlarging Convex Free Regions (CFRs)} \label{sec:max_fr}

A \ac{fr} $\freeRegion$ is formed around a center point $\cen \in \free$ with radius $r = \sd(\cen)$.
If $\cen$ approaches an obstacle $r$ shrinks and in the limit ($r = \sd(\cen) =0$) the \ac{fr} collapses to a point.
Such situations result in very restrictive constraints ($\norm{\pos - \cen}_\nidx \leq r$).
To avoid this problem, we grow \acp{fr} as proposed by \cite{schoels2019nmpc}.
% and generalize their approach to arbitrary norms.

\begin{figure}
	\centering
	\includegraphics[width=\columnwidth]{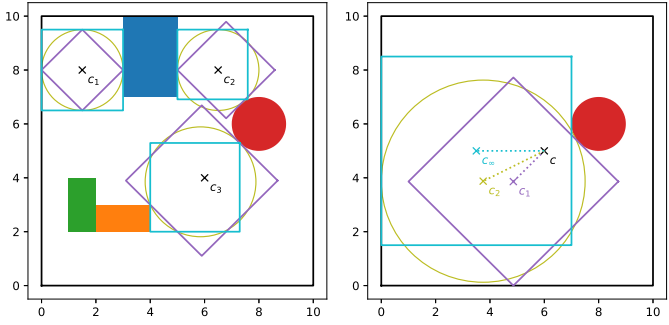}
	\caption{Enlarged \acfp{fr}. Left: enlarged regions for $L_1$ (purple), $L_2$ (olive), and $L_\infty$ (cyan) for the starting points and environment from Fig.~\ref{fig:sdf_example}.
	Right: the same for a simplified environment, including corresponding center points and search directions.}
	\label{fig:mfr_example}
\end{figure}

\begin{assumption}
	Assume that $\cen \in \free$ and that $\sd(\cen)$ is defined $\forall \cen \in \R^n$.\footnote{If $\cen \notin \free$ we follow the gradient to find a $\cen^\prime \in \free$.}
\end{assumption}
We then find a larger \ac{fr} $\freeRegion[\cen^*]$ via line search along a search direction $\mathbf{g} = \normalized[\nidx]{\nabla_{\cen} \sd(\cen)}$, i.e., the \ac{sdf}'s normalized gradient\footnote{It is sufficient to implement $\nabla_{\cen} \sd(\cen)$ using finite differences.}.
%Note that $\norm{\nabla_{\cen} \sd(\cen)}_\nidx = 1$ holds almost everywhere, except for points where it is undefined, e.g. ridges.
It is defined almost everywhere, except for points where $\sd(\cen)$ is undefined,e.g. ridges.
In the latter case we stop the search immediately.
For a given initial point $\cen \in \free$, we compute the optimal step size $\eta^*$ by solving
\begin{equation}\label{eq:grow_cfr}
\eta^* = \arg \underset{\eta\geq 0}{\max} \;\; \eta \;\quad \mathrm{s.t.} \quad \sd(\underbrace{\eta \cdot \mathbf{g} + \cen}_{=\cen^*}) = \eta + \sd(\cen),
\end{equation}
We obtain an enlarged \ac{fr} $\freeRegion[\cen^*]$ with center $\cen^* = \eta^* \cdot \mathbf{g} + \cen$ and radius $r^*=\sd(\cen^*)$. Note that the line search terminates on a ridge.
Figure~\ref{fig:mfr_example} shows free regions produced by this line search approach for different norms.

Solving \eqref{eq:grow_cfr} yields a new free region $\freeRegion[\cen^*]$ that includes the original one $\freeRegion$.
When solving this problem numerically we introduce a small tolerance to allow for numerical errors.
%errors introduced by floating point arithmetic.

\begin{lemma} \label{lem:grow_cfr}
	If $\cen \in \free$, $\mathbf{g} \in \{g \in \R^n : \norm{g}_\nidx = 1 \} $, and $\eta \geq 0$, with $\cen^* = \eta \cdot \mathbf{g} + \cen$ and
	$\sd(\cen^*) = \eta + \sd(\cen)$ then $\freeRegion \subseteq \freeRegion[\cen^*]$.
\end{lemma}
\begin{proof}
	We prove this by contradiction, assuming $\exists \; \pos \in \freeRegion$ such that $\pos \notin \freeRegion[\cen^*]$.
	Using \eqref{eq:freeRegion} we rewrite our assumption to $\norm{(\eta \cdot \mathbf{g} + \cen) - \pos}_\nidx > \sd(\cen^*)$. Applying the triangle inequality on the left side yields $\norm{\pos - (\cen + \eta \cdot \bf{g})}_\nidx \leq \norm{\pos - \cen}_\nidx + \norm{\eta \cdot \bf{g}}_\nidx = \norm{\pos - \cen}_\nidx + \eta$ and based on our assumption $\norm{\pos - \cen}_\nidx + \eta \leq \sd(\cen) + \eta$ holds. Inserting this gives $\sd(\cen) + \eta > \sd(\cen^*)$ and thus contradicts the condition $\sd(\cen^*) = \eta + \sd(\cen)$.
\end{proof}
\begin{remark}
	This lemma is a generalization of Lem.~6 by \cite{schoels2019nmpc} to arbitrary norms $\norm{\cdot}_\nidx$.
\end{remark}

\subsection{Discrete Time OCP} \label{sec:discrete_time}

We discretize the \ac{ocp} \eqref{eq:ocp} by splitting the horizon into $N+1$ time steps $t_k = k \cdot \Delta t$ for $k=0, \ldots, N$ and a chosen sampling time $\Delta t$. For more compact notation, we use the shorthand $\state_k = \state(t_k)$ to denote discrete time quantities.
\begin{assumption} \label{ass:constant_controls_steady_state_at_0}
	We use piece-wise constant controls, i.e., $\controls(t) = \controls_k\, \forall t\in [t_k, t_{k+1})$, and assume that $\terminalstate$ is a steady state at $\controls=0$, i.e., $0 = A \terminalstate$.
\end{assumption}

\subsubsection{Objective function:}
We approximate time optimal behavior without time scaling, i.e., for a fixed $\Delta t$, using the stabilizing scheme proposed by \cite{Verschueren2017a}.
They show that for the case of point-to-point motion, i.e., if the robot shall move from $\inistate$ to $\terminalstate$ in minimal time, the time optimal objective function can be approximated by
\begin{align*}
 \underset{\begin{subarray}{c}
	\state_0, \ldots, \state_N,\\
	\controls_0, \ldots, \controls_{N -1}
	\end{subarray}}{\min} & \sum_{k=0}^{N-1} \alpha^k\norm{\state_k-\state_\mathrm{g}}
\end{align*}
with initial condition $\state_0 = \inistate$ and terminal constraint $\state_N = \terminalstate$ for both $N$ and $\alpha > 1$ large enough, such that time optimality is recovered.
Note that this transformation is norm-agnostic.

\subsubsection{Robot Model:}
The piece-wise constant controls allow us to discretize the dynamical model using the matrix exponential. This yields
$A_\D = e^{A\Delta t}$ and $B_\D =\left(\int_{0}^{\Delta t} e^{A t} \mathrm{d}t \right) B$
%In general $B_\D = \left( \int_{0}^{\Delta t} e^{A \tau} \mathrm{d} \right) B$
and the discrete dynamics $\state_{k+1} = A_\D \state_k + B_\D \controls_k.$

\subsubsection{Occupied Set:}
In Sec.~\ref{sec:collision_avoidance} describes a convex inner approximation of the actual collision constraint for an arbitrary occupied set $\occupied$.
We use this formulation to derive a collision avoidance constraint that can accommodate arbitrary time dependent occupied sets $\occupied(t)$ and can thus account for predictable dynamic obstacles.

The discrete time occupied set $\occupied_k$ is defined as the union of all occupied sets $\occupied(t)$ in the time interval $[t_k, t_{k+1}]$:
\begin{equation} \label{eq:occupied_k}
	\occupied_k = \bigcup_{t=t_k}^{t_{k+1}} \occupied(t) .
\end{equation}

Obviously the continuous time occupied set $\occupied(t)$ is contained in the discrete time occupied set $\occupied_k\, \forall t \in [t_k, t_{k+1}]$.
\begin{assumption}
Assume that the discrete time occupied set $\occupied_k$ is known for all $k=0,\ldots,N$. This assumption is realistic as mobile robots typically possess systems to estimate motions and states of surrounding agents.
\end{assumption}

\subsubsection{Collision Avoidance Constraint:}
To guarantee collision avoidance in continuous time, the robot's movement needs to be accounted for. For more compact notation, we define the set of all points that are occupied by the robot in the time interval $[t_k, t_{k+1}]$ as $\Omega_k = \bigcup_{t=t_k}^{t_{k+1}} \robot(\state(t))$ and introduce the shorthand $\robot_k = \robot(x_k)$.
We define the robot's action radius for all $k=0,\ldots, N$ as
\begin{equation}
	\rho_k = \underset{\pos \in \Omega_k}{\max} \; \dist[\robot_k](\pos) = \underset{\pos \in \Omega_k}{\max} \; \underset{\pos^\prime \in \robot_k}{\min} \norm{\pos - \pos^\prime}_\nidx .
\end{equation}
%The robot's action radius $\rho_k$ is an upper bound for the displacement of points in the set $\robot_k$ (the set of points occupied by the robot at $t_k$) during the time interval $[t_k, t_{k+1}]$.

The \emph{Taylor upper bound} introduced in Appendix~\ref{sec:taylor_ub}, can be used to compute an upper bound for the robot's action radius.
The robot's position $\pos(t)$ can be rewritten as $\pos(t) = \pos(t_k) + (\pos(t) - \pos(t_k)) = \pos_k + \Delta \pos(t, t_k)$.
If the first $m$ derivatives of $\pos$ with respect to time are known and that the $m$\textsuperscript{th} derivative is globally bounded by $\norm{\pos^{(m)}(t)}_\nidx \leq \ub{p}^{(m)}, \forall t \in \R$, then the \emph{Taylor upper bound} yields
\begin{equation}
\norm{\Delta \pos(t, t_k)}_\nidx \leq \sum_{i=1}^{m-1} \norm{\pos^{(i)}_k}_\nidx \frac{\Delta t^i}{i!} + \ub{p}^{(m)} \frac{\Delta t^m}{m!}.
\end{equation}
\begin{assumption} \label{ass:bounded_derivatives}
	Assume that the first $m-1$ derivatives $\pos^{(1)}_k, \ldots, \pos^{(m-1)}_k$ for $k=0,\ldots,N$ are known and point wise bounded, such that $\Vert\pos^{(i)}_k \Vert_\nidx \leq \bar{p}^{(i)}$ for $i = 1,\ldots, m-1$.
\end{assumption}
This yields a global upper bound $\rho$ of the action radius $\rho_k$:
\begin{equation} \label{eq:action_radius}
\norm{\Delta \pos(t, t_k)}_\nidx \leq \rho := \sum_{i=1}^{m} \ub{p}^{(i)} \frac{\Delta t^i}{i!},
\end{equation}
for all $k=0,\ldots,N$. Note that $\rho$ is independent of the time index $k$, but requires that all $\pos^{(i)}_k$ are bounded.

Using Assumption~\ref{ass:robot_hull}, we approximate $\robot_k$ by a convex bounding polytope with vertices $\ub{\robot}_k = \{ \nu_{1, k}, \ldots, \nu_{\mathrm{n_\robot}, k} \} $.
% To achieve full-body collision avoidance, we assume that $\robot_k$ can be approximated by a convex bounding polytope with vertices $\ub{\robot}_k = \{ \nu_{1, k}, \ldots, \nu_{\mathrm{n_\robot}, k} \} $, as in Sec.~\ref{sec:collision_avoidance}.
\begin{assumption} \label{ass:linear_motion}
	For simplicity, assume that the robot is constrained to translational movement, such that the vertices are given by $\nu_{i, k} = S_{\pos} \cdot \state_k + l_i$, where $S_{\pos}$ is a selector matrix, such that $\pos = S_{\pos} \cdot \state$ is the robot's position.
\end{assumption}
% Note that CIAO$^\star$'s applicability to nonlinear systems has been demonstrated in \cite{schoels2019nmpc}.
% Due to the linear motion assumption, the vertices are given by $\nu_{i, k} = S_{\pos} \cdot \state_k + l_i$, where $S_{\pos}$ is a selector matrix, such that $\pos = S_{\pos} \cdot \state$ is the robot's position.
Assumption \ref{ass:linear_motion} yields a convex full-body collision avoidance constraint
for $i = 1, \ldots, n_\robot$ and $k = 0, \ldots, N$
\begin{equation} \label{eq:discrete_collision_avoidance_constraint}
	\norm{\nu_{i, k} - \cen_k }_\nidx \leq \sd[\occupied_k](\cen_k) - \rho.
\end{equation}
Note that for a fixed $\cen_k$ \eqref{eq:discrete_collision_avoidance_constraint} is a conic constraint in $\nu_{i, k}$ and thereby a set of conic constraints in $\state_k$.
Lem.~\ref{lem:ctca} implies that \eqref{eq:discrete_collision_avoidance_constraint} guarantees continuous time collision avoidance.
\begin{lemma}\label{lem:ctca}
	Given the occupied set $\occupied_k$, the robot action radius $\rho \geq 0$, and a convex bounding polytope with vertices $\ub{\robot}_k = \{ \nu_{1,k}, \ldots, \nu_{\mathrm{n_\robot},k} \}$ such that $\robot_k \subseteq \mathrm{convhull}(\ub{\robot}_k)$, then
	\begin{equation*}
		\norm{\nu_{i,k} - \cen_k}_\nidx \leq \sd(\cen_k) - \rho, \quad i = 1, \ldots, n_\robot
	\end{equation*}
	implies $\interior(\robot(\state(t))) \cap \occupied(t) = \emptyset$ for all $t \in [t_k, t_{k+1}]$.
\end{lemma}
\begin{proof}
	We prove this lemma by showing that all vertices $\nu_{i}(t)$ for $t \in [t_k, t_{k+1}]$ and $i = 1, \ldots, n_{\robot}$ are inside the \acl{fr} $\freeRegion[\cen_k]$. First, we note $\norm{\nu_{i,k} - \cen_k}_\nidx \leq \sd(\cen_k) - \rho \Leftrightarrow \norm{\nu_{i,k} - \cen_k}_\nidx + \rho \leq \sd(\cen_k)$.
	Further, the distance between $\nu_i(t)$ and $\cen_k$ is given by $\norm{\nu_{i}(t) - \cen_k}_\nidx = \norm{S_{\pos}\state_k + l_i + \Delta \pos(t, t_k) - \cen_k}_\nidx$. Applying the triangle inequality yields $\norm{\nu_{i}(t) - \cen_k}_\nidx \leq \norm{S_{\pos}\state_k + l_i - \cen_k}_\nidx + \norm{\Delta \pos(t, t_k)}_\nidx$. From \eqref{eq:action_radius} we know $\norm{\Delta \pos(t, t_k)}_\nidx \leq \rho$, which results in $\norm{\nu_{i}(t) - \cen_k}_\nidx \leq \norm{\nu_{i,k} - \cen_k}_\nidx + \rho \leq \sd(\cen)$
%	The robot's position is given by $\pos(t) = S_{\pos} \state(t)$ and it can be decomposed into $\pos(t) = \pos_k + \Delta \pos_k (t)$.
%	For all $i$ holds $\nu_{i}(t) = \pos(t) + l_i = \pos_k + l_i + \Delta \pos_k(t)$ and therefore $\norm{\nu_{i}(t) - \cen_k}_\nidx = \norm{ \pos_k + l_i + \Delta \pos_k(t) - \cen_k}_\nidx \leq \norm{ \pos_k + l_i - \cen_k}_\nidx + \norm{\Delta \pos_k (t)}_\nidx$ by the triangle inequality. From \eqref{eq:action_radius} we know $\norm{\Delta \pos_k (t)}_\nidx \leq \rho_k \forall t \in [t_k, t_{k+1}]$, which yields $\norm{\nu_{i,k} - \cen}_\nidx + \rho_k \leq \sd(\cen) $.
\end{proof}

\begin{remark}
	This guarantee can be extended to arbitrary motion, by including an upper bound for the robot's displacement due to rotation, c.f. \cite{Schulman2014}.
\end{remark}

\subsubsection{Path Constraints:}
To obtain continuous time constraint satisfaction for the convex polyhedral set $\set{H}$ we utilize the \emph{Taylor upper bound} (see App.~\ref{sec:taylor_ub}).
This results in a smaller, convex polyhedral set.
In addition, we impose the constraints resulting from Ass.~\ref{ass:bounded_derivatives}.
The resulting discrete time path constraints form a convex polytope $\set{H}_\D$.

\subsection{The CIAO$^\star$-NLP}
Applying the reformulations detailed in the Sec.~\ref{sec:discrete_time} to \eqref{eq:ocp}, we obtain a \ac{nlp}, the CIAO$^\nidx$-NLP. It is a convex conic problem that depends on goal state $\terminalstate$, the initial state $\inistate$, the sampling time $\Delta t$, the tuple of \ac{fr} center points $C = \left( \cen_0, \ldots, \cen_N \right) $, and the corresponding radii $r_k = \sd[k](\cen_k) - \rho$, where $\rho$ is the robot's action radius \eqref{eq:action_radius}.
The CIAO$^\nidx$-NLP is given by
{\small
	\begin{mini!}
		{\vec{w}}% optimization variable
		{\sum_{k=0}^{N-1} \alpha^k\norm{\state_k-\state_\mathrm{g}}_{Q_\mathrm{x}} } % objective Function and label for the objective
		{\label{eq:lp}} % label of the optimization problem
		{} % result of optimization, e.g. J^* =
		%\addConstraint{LHS.1}{RHS.1\label{Const1}}{extraConst1}
		\addConstraint{\state_0}{= \inistate}{}
		\addConstraint{\state_N}{= \terminalstate}{}
		\addConstraint{\state_{k+1}}{= A_\D \state_k + B_\D \controls_k ,}{\quad k =0,\ldots,N-1}
		\addConstraint{(\state_k, \controls_k)}{\in \set{H}_\D ,}{\quad k =0,\ldots,N}
		\addConstraint{\hspace{-0.7cm}\norm{ S_{\pos}\state_k + l_i - \cen_k}_\star}{\leq r_k, \; i=1,\ldots,n_\robot, \label{eq:lp_cac}}{\quad k =0,\ldots,N,}
	\end{mini!}
}%
where $\alpha > 1$ and $N$ are both large enough such that time optimal behavior is recovered \citep{Verschueren2017a}.
Here $\vec{w} = \left[\state_0\T, \controls_0\T, \ldots, \controls_{N-1}\T, \state_N\T \right]\T$ is the vector of optimization variables.
CIAO$^\nidx$ is norm agnostic, but for the sake of clear notation we introduce $Q_\mathrm{x}$ as a norm specifier.
The matrices $A_\D, B_\D$ result from the discretization of the model. The convex set $\set{H}_\D$ guarantees continuous time constraint satisfaction of the original path constraints, we choose $\controls_N = 0$ in accordance with Assumption~\ref{ass:constant_controls_steady_state_at_0}.
The collision avoidance constraint \eqref{eq:lp_cac} is a reformulation of \eqref{eq:discrete_collision_avoidance_constraint} that uses Assumption~\ref{ass:linear_motion}.

Note that \eqref{eq:lp} is a \ac{lp} if $Q_\mathrm{x}, \nidx \in \{1, \infty\}$ and can be solved efficiently, as demonstrated in Sec.~\ref{sec:experiments}.
% This property is remarkable, because \acp{lp} can be solved efficiently, as demonstrated in Sec.~\ref{sec:experiments}.

\subsection{The CIAO$^\nidx$-Iteration} \label{sec:ciao_iteration}

The \emph{CIAO$^\nidx$-iteration} computes a new trajectory $\vec{w}^*$ for a provided initial guess $\vec{w}$ as described in Alg.~\ref{alg:ciao_iteration}.
\begin{algorithm}
	\small
	\begin{algorithmic}[1]
		\Function{CIAO$^\nidx$-iteration}{$\vec{w}\, ; \; \terminalstate, \; \inistate, \; \Delta t$}
		\State $C \gets (\cen_k = S_{\pos} \cdot \state_k$ for $k =0,\ldots,N)$
		\State $C^* \gets (\cen^* = \textsc{grow\acs{fr}}(\cen)$ for all $\cen \in C$) \Comment solve \eqref{eq:grow_cfr}
		%			\priority{\Ensure $\dist(\cen) > \dmin \; \forall \, \cen \in \set{C}$}
		\State $\vec{w}^* \gets$ \textsc{solveCIAO$^\nidx$-NLP}($\vec{w} ; \; C^*, \; \terminalstate, \; \inistate, \; \Delta t$) \Comment{solve \eqref{eq:lp}}
		\EndFunction \ \Return $\vec{w}^*$ \Comment{return newly found trajectory}
	\end{algorithmic}
	\caption{the CIAO$^\nidx$-iteration}
	\label{alg:ciao_iteration}
\end{algorithm}\\
First, Alg.~\ref{alg:ciao_iteration} obtains a tuple of center points $C$ using the initial guess $\vec{w}$ (Line~2). Recall that the robot's position is given by $\pos_k = S_{\pos} \state_k$.
The center points $C$ are then optimized as described in Sec.~\ref{sec:max_fr} (Line~3). The CIAO$^\star$-NLP is then solved using a suitable solver (Line~4).

Note that the CIAO$^\nidx$-iteration preserves feasibility, i.e. if the guess $w$ is feasible, $w^*$ is feasible.
In the case of robot motion planning that means: for a kinodynamically feasible and collision free initial guess $\vec{w}$ Alg.~\ref{alg:ciao_iteration} finds a kinodynamically feasible and collision free trajectory $\vec{w}^*$ that is faster or equally fast.

\section{CIAO$^\nidx$ for Motion Planning} \label{sec:ciao_for_motion_planning}

This section describes the application of the CIAO$^\star$-iteration (see Sec.~\ref{sec:ciao_iteration}) for motion planning. We propose two algorithms: one for offline trajectory optimization and a second for online motion planning and control.

\subsection{CIAO$^\nidx$ for Trajectory Optimization}

%Now we consider the application of CIAO$^\star$ for offline trajectory optimization.
Alg.~\ref{alg:ciao_trajopt} iteratively optimizes trajectories and approximates the time optimal solution. Starting with an initial guess $\vec{w}$, it uses the CIAO$^\nidx$-iteration to improve the initial guess (Line~1).
\begin{algorithm}
	\small
	\begin{algorithmic}[1]
		\Require $\vec{w}, \inistate, \terminalstate, \Delta t, \varepsilon$ \Comment{initial guess, start and goal state}
		\State $\vec{w}^* \gets \textsc{CIAO$^\nidx$-iteration}(\vec{w}\, ; \; \terminalstate, \; \inistate, \; \Delta t)$ \Comment{see Alg.~\ref{alg:ciao_iteration}}
		\While {$\textsc{cost}(\mathbf{w}^*) - \textsc{cost}(\vec{w}) > \varepsilon$}
		\State $\mathbf{w} \gets \mathbf{w}^*$ \Comment{set last solution as initial guess}
		\State $\vec{w}^* \gets \textsc{CIAO$^\nidx$-iteration}(\vec{w}\, ; \; \terminalstate, \; \inistate, \; \Delta t)$ \Comment{see Alg.~\ref{alg:ciao_iteration}}
		\EndWhile
		\State \Return $\mathbf{w}^*$
	\end{algorithmic}
	\caption{CIAO$^\nidx$ for offline trajectory optimization}
	\label{alg:ciao_trajopt}
\end{algorithm}
Further CIAO$^\nidx$-iterations follow \mbox{(Line~3--4)}, until the improvement of the trajectory w.r.t. some cost function (e.g. CIAO$^\nidx$'s objective function) falls below a chosen threshold $\varepsilon$ (Line~2).

Note that Alg.~\ref{alg:ciao_trajopt} preserves feasibility, because the CIAO$^\nidx$-iteration does. Once it converges to a feasible trajectory, all further iterations yield feasible (and better) trajectories.

\subsection{CIAO$^\nidx$-MPC: Online Motion Planning}

Additionally, we propose CIAO$^\nidx$-MPC. This algorithm unifies trajectory optimization and tracking, also referred to as online motion planning.
To meet the time constraints of the robot's control loop, we use a horizon of fixed length $N$, which is typically shorter than the one used for trajectory optimization (as described before). As a consequence the goal might not be reachable during the now receding horizon. Therefore we replace the terminal constraint by a terminal cost and obtain an approximation of the CIAO$^\nidx$-NLP \eqref{eq:lp}, the CIAO$^\star$-MPC-NLP:
{\small
	\begin{mini}
		{\vec{w}}% optimization variable
		{\alpha_\mathrm{N} \norm{\state_N-\state_\mathrm{g}}_{Q_{\mathrm{x}}} + \sum_{k=0}^{N-1} \alpha^k\norm{\state_k-\state_\mathrm{g}}_{Q_{\mathrm{x}}} } % objective Function and label for the objective
		{\label{eq:lp_mpc}} % label of the optimization problem
		{} % result of optimization, e.g. J^* =
%		\breakObjective{ }
		%\addConstraint{LHS.1}{RHS.1\label{Const1}}{extraConst1}
		\addConstraint{\state_0}{= \inistate}{}
		\addConstraint{A \state_N}{=0}
		\addConstraint{\state_{k+1}}{= A_\D \state_k + B_\D \controls_k ,}{\quad k =0,\ldots,N-1}
		\addConstraint{(\state_{k}, \controls_k)}{\in \set{H}_\D ,}{\quad k =0,\ldots,N}
		\addConstraint{\hspace{-0.7cm}\norm{ S_{\pos}\state_k + l_i - \cen_k}_\star}{\leq r_k, \; i=1,\ldots,n_\robot,}{\quad k =0,\ldots,N,}
	\end{mini}
}%
where $\alpha_\mathrm{N} \gg \alpha^N$ is the terminal cost's scaling factor.
The terminal constraint $A \state_N = 0$ ensures that the robot comes to a full stop at the end of the horizon, i.e. it reaches a steady state at $\controls_N = 0$.
%By constraining the robot's velocity to be zero at the end of the horizon we prevent collisions that could occur when receding the horizon and get recursive feasibility.
This constraint prevents collisions that could occur when shifting the horizon forward in time.
The \ac{nlp} above replaces \eqref{eq:lp} in Line~4 of Alg.~\ref{alg:ciao_iteration}.

\begin{algorithm}
	\small
	\begin{algorithmic}[1]
		%\Require $\set{P} = \{\pos_0, \ldots, \pos_M\}$ \Comment{Some initial path}
		\Require $\vec{w},\; \terminalstate, \; \Delta t$ \Comment initial guess, goal state, and sampling time
		\While {$\inistate \neq \terminalstate$ } \Comment goal reached?
		\State $\inistate \gets \textsc{getCurrentState}()$
		\State $\vec{w}^* \gets \textsc{CIAO$^\nidx$-iteration}(\vec{w}\, ; \; \terminalstate, \; \inistate, \; \Delta t)$ \Comment{Alg.~\ref{alg:ciao_iteration} with \eqref{eq:lp_mpc}}
		\State \textsc{controlRobot}$(\vec{w}^*)$ \Comment apply control $\controls_0^*$
		\State $\vec{w} \gets \textsc{shiftTrajectory}(\vec{w}^*)$ \Comment recede horizon
		\EndWhile
	\end{algorithmic}
	\caption{CIAO$^\nidx$-MPC}
	\label{alg:ciao_mpc}
\end{algorithm}
Alg.~\ref{alg:ciao_mpc} formally introduces the CIAO$^\nidx$-MPC method.
While the robot is not at the goal $\terminalstate$ (Line~1), we obtain the robot's current state $\inistate$ (Line~2). Next we formulate and solve \eqref{eq:lp_mpc} as described in Alg.~\ref{alg:ciao_iteration} (Line~3). We conclude each iteration by controlling the robot (Line~4) and shifting the trajectory $\vec{w}$ (Line~5).

Under mild assumptions CIAO$^\star$-MPC has recursive feasibility, refer to Appendix~\ref{sec:recursive_feasibility} due to space constraints.\footnote{The Appendix is available at \texttt{https://arxiv.org/abs/2001.05449}.}

\section{Experiments} \label{sec:experiments}

The performance and resulting behavior of CIAO$^\nidx$ was investigated experimentally. A first set of experiments compares trajectories found by CIAO$^\nidx$ to a time optimal reference.
In a second set of experiments we compare the behavior of CIAO$^\nidx$ for $\nidx \in \{1, 2,\infty\}$.
A final set of experiments evaluates CIAO$^\nidx$-MPC's behavior.
In all experiments we consider a circular robot with puck dynamics that is controlled through its jerks ($\nstate = 6, \ncontrols = 2$) with
{\small
\begin{align}
	A_\mathrm{D} = \begin{bmatrix}
		\Id_2 & \Delta t \cdot \Id_2, & \Delta t^2 / 2 \cdot \Id_2\\
		& \Id_2 & \Delta t \cdot \Id_2\\
		& & \Id_2
	\end{bmatrix},
	&&
	B_\mathrm{D} = \begin{bmatrix}
		\Delta t^3 / 6 \cdot \Id_2\\
		\Delta t^2 / 2 \cdot \Id_2\\
		\Delta t \cdot \Id_2
	\end{bmatrix}.
\end{align}}%
The reasons for this choice are twofold:
(1) the lemmata above consider linear dynamics, the applicability of CIAO$^2$ to nonlinear robots and dynamic environments has been demonstrated by \cite{schoels2019nmpc}, and
(2) nonlinear dynamics would blur the comparison of different norms.
Note that a more complex robot shape would not affect the number of distance function evaluations, but only increase the number of constraints in the \acp{nlp}.

The initial trajectories were computed from a path found by RRT.
All \acp{nlp} were formulated as direct multiple-shooting in JuMP \citep{Dunning2017} and solved by Gurobi \citep{gurobi} on an Intel Core i7-8559U clocked at $2.7\, \rm{GHz}$ running macOS Mojave.

\subsection{Evaluation of time optimality} \label{sec:time_optimal_evaluation}

To evaluate the quality of CIAO$^\nidx$'s approximation of time optimal behavior, it has been tested in $50$ scenarios filled with 5 circular, randomly placed obstacles with radii between $1$ and $2\,\mathrm{m}$. The robot has to move from a point in the lower left of the environment through the obstacles to a point in the upper right corner.
\begin{table}
	\centering
	\begin{tabular}{r|cccc}
\toprule
               & min   & mean  & median & max   \\ \hline
  time to goal [ratio] & 1.015 & 1.024 &  1.020 & 1.038 \\
   path length [ratio] & 0.984 & 0.999 &  1.000 & 1.074 \\
control effort [ratio] & 0.925 & 0.984 &  0.998 & 1.011 \\
     clearance [ratio] & 0.999 & 0.999 &  0.999 & 0.999 \\
\bottomrule
\end{tabular}
	\caption{Evaluation of time optimal behavior approximation. We compute the normalized performance measures, obtained by taking the ratio between CIAO$^2$'s solution and the time optimal reference (see examples in Fig.~\ref{fig:to_traj}).}
	\label{tab:time_optimal_eval}
\end{table}
The values reported in Table~\ref{tab:time_optimal_eval} were obtained for CIAO$^2$.
We note that CIAO$^2$ finds close to time optimal trajectories. In the considered scenarios it is always less than $4\%$ slower.
One reason, why CIAO$^2$ does not fully converge to the time optimal solution can be seen in the bottom left plot of Fig.~\ref{fig:to_traj}. The circular boundaries of the \acp{fr} prevent motion on the safety margin (denoted by the red lines) like done by the reference in that case.
We note that the average path length and control effort is even lower than the time optimal solution, meaning that CIAO$^\nidx$ finds shorter trajectories that require less control activation, but is slightly slower than the time optimal solution.
At the same time it maintains the same or similar minimum distance to all obstacles, reported as clearance.

\subsection{Comparison of CIAO$^1$, CIAO$^2$, and CIAO$^\infty$}
\begin{figure}
	\centering
	\includegraphics[width=\columnwidth]{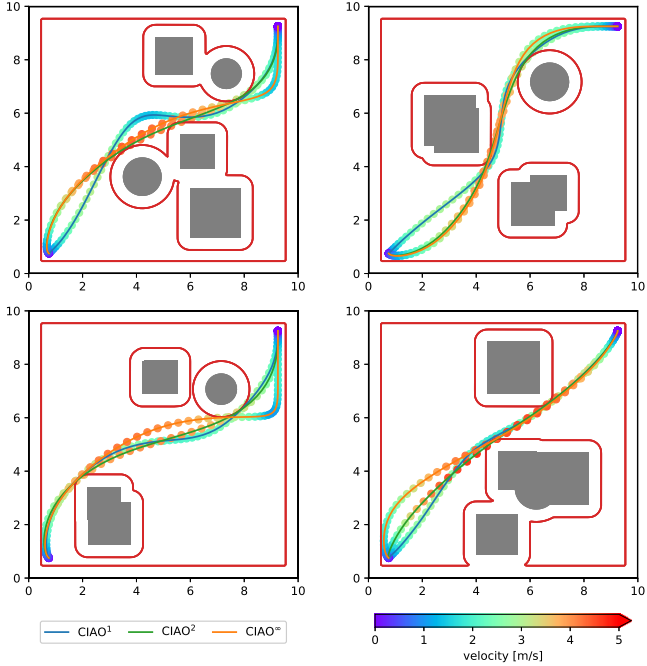}
	\caption{CIAO$^\star$ trajectories for different norms (values of $\nidx$): CIAO$^1$ (blue), CIAO$^2$ (green), and CIAO$^\infty$ (orange).
		For CIAO$^2$ the safety margin to obstacles $\rho$ is marked in red. Note that the safety margins for CIAO$^1$ \& CIAO$^\infty$ take different shapes, as depicted in Fig~\ref{fig:sdf_example}.}
	\label{fig:norm_illustration}
\end{figure}%
For these experiments we use a similar setup as for the time optimal one. The only difference is that the obstacles are now a mixture of circles and rectangles (both with the same probability).
The experimental results reported in Tables~\ref{tab:norm_comparison}~\&~\ref{tab:norm_numerical_comparison} were obtained on $50$ simulated scenarios.

\begin{table}
	\centering
	\begin{tabular}{r|ccc}
\toprule
                 & CIAO$^1$ & CIAO$^2$ & CIAO$^\infty$ \\ \hline
    success rate &    96\% &    \textbf{100}\% &         92\% \\
time to goal [s] &    5.80 (7.40)&    \textbf{5.70} (\textbf{6.30})&         5.80 (7.20)\\
 path length [m] &   \textbf{13.569} &   13.582 &        13.948 \\
                 & (15.338) & (\textbf{15.249}) &      (15.429) \\
   clearance [m] &    0.294 &    0.262 &         0.262 \\
\bottomrule
\end{tabular}
	\caption{Trajectory quality comparison for CIAO$^1$, CIAO$^2$, and CIAO$^\infty$. We report the median (and maximum) seconds and meters in simulation (see examples in Fig.~\ref{fig:norm_illustration}) of the successful scenarios.}
	\label{tab:norm_comparison}
	\begin{tabular}{r|ccc}
\toprule
                       & CIAO$^1$  & CIAO$^2$ & CIAO$^\infty$ \\ \hline
   processing time [s] &     1.709 &    1.079 &         \textbf{0.656} \\
                       &  (23.033) &  (5.598) &       (\textbf{3.182}) \\
            iterations &    26 (398) &    \textbf{8} (\textbf{45})&        10.5 (51) \\
time per iteration [s] &     0.070 &    0.127 &         \textbf{0.064} \\
                       &   (0.112) &  (0.193) &       (\textbf{0.098}) \\
iterations to feasible &     2 (5) &    2 (\textbf{4}) &         2 (6) \\
\bottomrule
\end{tabular}
	\caption{Computational effort of CIAO$^1$, CIAO$^2$, and CIAO$^\infty$ for the same experiments as in Table~\ref{tab:norm_comparison}.
		We report the median (and maximum) computation time and iterations of the successful cases.}
	\label{tab:norm_numerical_comparison}
\end{table}%
% trajectory quality
Looking at the trajectory quality comparison in Table~\ref{tab:norm_comparison}, we observe that CIAO$^1$, CIAO$^2$, and CIAO$^\infty$ obtain similar results.
CIAO$^2$ finds the fastest trajectories, while the shortest ones are found by CIAO$^1$.
The reasons for these findings are visible in Fig.~\ref{fig:norm_illustration}.
In comparison CIAO$^1$ takes a more direct route at a lower average speed.
It also maintains the highest clearance (measured by the Euclidean distance to the closest obstacle) due the diamond shape of the \acl{fr}, which is tied to the $L_1$ norm.
The diamonds are quite restrictive for diagonal movement, but become comparatively large in proximity of corners allowing for smooth maneuvering around corners (see Fig.~\ref{fig:sdf_example} \& \ref{fig:mfr_example}, $c_3$).
The $L_\infty$ norm on the other hand finds large regions in tunnels and corridors (see Fig.~\ref{fig:sdf_example}, $c_1$) and is preferred for diagonal movement, but it is restrictive in proximity of corners (see Fig.~\ref{fig:sdf_example} \& \ref{fig:mfr_example}, $c_2$ \& $c_3$) which can result in detours (see Fig.~\ref{fig:norm_illustration}).
The failures of CIAO$^1$ and CIAO$^\infty$ result from passages that were to narrow to accommodate their \acp{fr}.

%Second, the diagonal transition and the orientation of obstacles (all aligned with the $x$-axis) favors the $L_\infty$ norm (see Fig.~\ref{fig:sdf_example}). This is intentional to benchmark how the shape of obstacles affects the performance of CIAO$^1$- and CIAO$^\infty$-NLP.
%Note that if the environment was rotated by 45 degrees, the behavior of CIAO$^1$ would be similar to the one of CIAO$^\infty$ and vise versa. CIAO$^2$, on the other hand, would be unaffected, because it is insensitive to the environments rotation (see Fig.~\ref{fig:sdf_example}).

% compuational effiency
In terms of computational efficiency CIAO$^\infty$ reaches the best performance. This has two reasons: First, the CIAO$^1$- and CIAO$^\infty$-NLPs are \aclp{lp}, while the CIAO$^2$-NLP is a \ac{socp}. The latter requires more computation time to solve resulting in a higher time per iteration. Second, CIAO$^\infty$ needs fewer iterations than CIAO$^1$.
We note that all algorithms typically find a feasible solution within the first $2$ iterations and CIAO$^2$ takes at most $4$ iterations, which indicates fast convergence.

% Summary
In summary CIAO$^2$ finds the fastest trajectories and is better suited for cluttered environments. CIAO$^1$ and CIAO$^\infty$ reach lower computation times per iteration, but are sensitive to the orientation and shape of obstacles. In our experiments they fail to steer the robot through some narrow passages due to the shape and size of their \acp{fr}.

\subsection{CIAO$^\nidx$-MPC Evaluation}
\begin{table}
	\centering
	\begin{tabular}{r|ccc}
\toprule
               & CIAO$^1$ & CIAO$^2$ & CIAO$^\infty$ \\ \hline
  success rate &    100\% &    100\% &         100\% \\
  time to goal [ratio] &    1.193 &  \textbf{  1.037} &         1.089 \\
               &  (2.246) &  (\textbf{1.150}) &       (1.463) \\
   path length [ratio] &    1.014 &    \textbf{0.993} &         1.014 \\
               &  (1.334) &  (\textbf{1.084}) &       (1.163) \\
control effort [ratio] &    1.012 &    \textbf{0.992} &         1.016 \\
               &  (1.428) &  (\textbf{1.076}) &       (1.267) \\
     clearance [ratio] &    \textbf{1.208} &    1.004 &         1.205 \\
               &  (\textbf{1.955}) &  (1.231) &       (1.917) \\
\bottomrule
\end{tabular}
	\caption{Trajectory quality evaluation. We report the median (and maximum) of normalized ratios between the solutions found by CIAO$^\nidx$-MPC and the time optimal reference.}
	\label{tab:receding_horizon_trajectory_performance}
	\begin{tabular}{r|ccc}
\toprule
                       & CIAO$^1$    & CIAO$^2$ & CIAO$^\infty$ \\ \hline
   processing time [s] &       0.058 &    0.073 &         \textbf{0.043} \\
                       &     (0.124) &  (0.095) &       (\textbf{0.050}) \\
       solver time [s] &   4.514e-04 &    0.017 &     \textbf{4.306e-04} \\
                       & (\textbf{7.024e-04}) &  (0.026) &   (9.505e-04) \\
\bottomrule
\end{tabular}
	\caption{Computational effort of CIAO$^\nidx$-MPC for the same experiments as in Table~\ref{tab:receding_horizon_trajectory_performance}.
		We report the median (and maximum) computation time per MPC step.}
	\label{tab:receding_horizon_numerical_eval}
\end{table}%

To evaluate the trajectory quality lost due to the approximation introduced by CIAO$^\nidx$-MPC, we performed experiments on the same $50$ scenarios considered in Sec.~\ref{sec:time_optimal_evaluation}. The obtained results are reported in Tables~\ref{tab:receding_horizon_trajectory_performance} \&~\ref{tab:receding_horizon_numerical_eval}.
We use a horizon of $50$ steps resulting in a total of 406 optimization variables (plus slacks).

We observe that the trajectories found by CIAO$^2$-MPC get closest to the time optimal reference and that they are at most $15\%$ slower. CIAO$^1$ and CIAO$^\infty$ on the other hand find slower trajectories, due to some detours induced by the shapes of their \acp{fr}. This behavior shows effect in all the path length, the time to goal and the clearance.
%\subsection{CIAO$^\nidx$ MPC in dynamic environments}
%\todo{not yet implemented}

All, CIAO$^{\{1,2,\infty\}}$, reach processing times shorter than $125\,\mathrm{ms}$ per MPC step (see Table~\ref{tab:receding_horizon_numerical_eval}). Only a small fraction of this time is required for solving the CIAO$^\star$-NLP, most of it is consumed by an inefficient implementation of the distance function used for solving \eqref{eq:grow_cfr}.
For CIAO$^{\{1, \infty\}}$ the CIAO$^\star$-MPC-NLP \eqref{eq:lp_mpc} is a \ac{lp}, which is solved in $< 1\,\mathrm{ms}$.

\section{Conclusion} \label{sec:conclusion}

% Summary
This paper presents CIAO$^\nidx$, a generalization of CIAO by \cite{schoels2019nmpc} to predictable dynamic environments and arbitrary norms that approximates time optimal behavior. Evaluations in simulation show that CIAO$^\nidx$ finds close to time optimal trajectories.

% Outlook
Future research will investigate a theoretical guarantees for nonlinear systems and unpredictable environments. It will also include a comparison to competing approaches and further study the convergence properties of CIAO$^\nidx$.

% \pagebreak
\balance
\bibliography{syscop,additions} % bib file to produce the bibliography
% with bibtex (preferred)

\appendix

\section{Taylor upper bound} \label{sec:taylor_ub}

We want to guarantee continuous time constraint satisfaction for constraints that take the form $\norm{p(t) - c } \leq r$, where $p(t)\in\R^n$ is an $m$-times differentiable function w.r.t. $t\in\R$ and $r \in \R$, $c\in\R^n$ are constant.
This problem can be approached by deriving an upper bound for the expression $\norm{p(t) - c }$.
In a first step we take the Taylor expansion of $p$ around the point $\bar{t}$:
\begin{equation} \label{eq:taylor_expansion}
	p(t) = p(\bar{t}) + \underbrace{
		\sum_{i=1}^{m-1} p^{(i)}(\bar{t}) \frac{(t-\bar{t})^i}{i!} + p^{(m)}(\tilde{t}) \frac{(t-\bar{t})^m}{m!}
	}_{=\Delta p(t; \bar{t}) \text{ with } \tilde{t} \in [\bar{t}, t]},
\end{equation}
where $p^{(i)}(\bar{t}) = \dpartial{^i p}{t^i}(\bar{t})$ for more compact notation.
Inserting \eqref{eq:taylor_expansion} into the initial expression and applying the triangle inequality we get
\begin{align}
	\norm{p(t) - c} &= \norm{p(\bar{t}) - c + \Delta p(t, \bar{t})} \\
		&\leq \norm{p(\bar{t}) - c } + \norm{\Delta p(t, \bar{t})}.
\end{align}
Under the assumption that the global upper bound of the $m$\textsuperscript{th} derivative of $p$ is known and given by $\bar{p}^{(m)} = \underset{t}{\max} \norm{p^{(m)}(t)}$,
we obtain an upper bound for $\norm{\Delta p(t; \bar{t})}$
\begin{align*}
	\norm{\Delta p(t; \bar{t})} &\leq \norm{\sum_{i=1}^{m-1} p^{(i)}(\bar{t}) \frac{(t-\bar{t})^i}{i!}} + \bar{p}^{(m)} \frac{\norm{(t-\bar{t})^m}}{m!}.
\end{align*}
Note that $\bar{p}^{(m)}$ can be interpreted as a Lipschitz-constant.
Assuming that $\norm{t-\bar{t}}$ is bounded by $\overline{\Delta t}$ and applying the triangle inequality simplifies this expression to
\begin{align}
\norm{\Delta p(t; \bar{t})} %&\leq \sum_{i=1}^{m-1} \norm{p^{(i)}(\bar{t}) \frac{(t-\bar{t})^i}{i!}} + \bar{p}^{(m)} \frac{\norm{(t-\bar{t})^m}}{m!}\\
&\leq \sum_{i=1}^{m-1} \norm{p^{(i)}(\bar{t})} \frac{\overline{\Delta t}^i}{i!} + \bar{p}^{(m)} \frac{\overline{\Delta t}^m}{m!}.
\end{align}
Finally this yields
\begin{align}
	\norm{p(t) - c} \leq \norm{p(\bar{t}) - c } + \underbrace{\sum_{i=1}^{m-1} \norm{p^{(i)}(\bar{t})} \frac{\overline{\Delta t}^i}{i!} + \bar{p}^{(m)} \frac{\overline{\Delta t}^m}{m!}}_{=\overline{\Delta p}(\bar{t}; p^{(1)}, \ldots, p^{(m-1)}, \overline{\Delta t})} .
\end{align}
Note that $\overline{\Delta p}(\bar{t}; p^{(1)}, \ldots, p^{(m-1)}, \overline{\Delta t})$ is an upper bound on $\norm{\Delta p(t, \bar{t})}$ that does not depend on $t$. It can thus be applied to reach continuous constraint satisfaction.

\section{Recursive Feasibility} \label{sec:recursive_feasibility}

\begin{assumption} \label{ass:tightening}
	We assume
	%that the occupied set is growing over the prediction horizon, i.e. $\occupied_k \subseteq \occupied_{k+1}, k = 0,\ldots,N$, and
	that the actual occupied set $\occupied_{\mathrm{a}}$ is a subset of the prediction $\occupied$, i.e. $\occupied_{\mathrm{a}}(t) \subseteq \occupied(t)\, \forall t \in \R$.
\end{assumption}

\begin{assumption} \label{ass:cooperative_agents}
	We assume that the occupied set at the end of the planning horizon contains all future occupied sets.
	This means
	$r_N = \sd[\occupied_N](\cen_N) - \rho \geq \bar{l} \Rightarrow r_{N+k} = \sd[\occupied_{N+k}](\cen_N) - \rho \geq  \bar{l}\, \forall k \in \mathbb{N}$, with $\bar{l} = \max_{\nu \in \robotHull_N} \norm{\nu - \cen_N}$.
	%We assume that other agents do not collide with a standing robot, i.e. if $S_{\vec{v}} \state_k = 0$ and $\sd[\occupied_k](\robotHull_k) \geq \rho$ then $\sd[\occupied_{k+1}](\robotHull_k) \geq \rho$ and $\robotHull_k = \robotHull_{k+1}$.
\end{assumption}
Recall that $\robotHull_N = \{ \nu_{1,N}, \ldots, \nu_{n_{\robot}, N} \}$ with $\nu_{i, N} = S_{\pos} \state_N + l_i$ for $i=1,\ldots,n_{\robot}$.

Assumption~\ref{ass:cooperative_agents} implies that the robot can stay stopped in its final position indefinitely, e.g. a car parked on a shoulder or parking spot.
% We assume that the the final position $\pos_N = S_{\pos} \state_N$ can always be chosen such that the robot can stay in it indefinitely and that no other agent or obstacle will collide with it, if that position was predicted to be free at the end of the planning horizon.
\begin{lemma}[Recursive Feasibility] \label{lem:recursive_feasibility}
	If Assumptions~\ref{ass:constant_controls_steady_state_at_0}, \ref{ass:tightening}, \ref{ass:cooperative_agents} hold and $\vec{w}$ is a feasible point of \eqref{eq:lp_mpc}, CIAO$^\star$-MPC finds feasible solutions $(w^*_i)_{i\in\mathbb{N}}$ in all further iterations.
\end{lemma}
\begin{proof}
	We will prove this Lemma using induction and show that for a feasible point $\vec{w}$, a point $\vec{w}^+$ exists which is a feasible point of \eqref{eq:lp_mpc} at the next iteration.
	% show that CIAO iteration preserves feasibility
	% Since the constraints in \eqref{eq:lp_mpc} are convex inner approximations of the actual constraints any feasible point $\vec{w}$ is a feasible point of the original \ac{ocp}.
	% feasibility of shift operation: Problem approaching agents.
	Thanks to Assumption~\ref{ass:constant_controls_steady_state_at_0} we know that $A x_N = 0$ with steady state control $\controls=0$ exists.
	This implies $\state_N = A_\D \state_N$.
	Therefore holds $\pos_N = S_{\pos} \state_{N} = S_{\pos} A_\mathrm{D} \state_N = S_{\pos} \state_{N+1} = \pos_{N +1}$ by construction.
	% Due to the collision avoidance constraint, we know $\norm{l_i} \leq \norm{\pos_N + l_i - \cen_N} \leq \sd[\occupied_N](\cen_N) - \rho$ for $i=1,\ldots,n_{\robot}$.
	Without loss of generality we choose $\cen_{N+1} = \pos_N$ and apply Assumption~\ref{ass:cooperative_agents} to get $\norm{\pos_N + l_i - \cen_N} = \norm{l_i} \leq \sd[\occupied_{N+1}](\pos_N) - \rho$ for $i=1,\ldots,n_{\robot}$.
	Thus $\vec{w}^+ = (\state_1, \controls_1, \ldots, \state_N, 0, \state_N)$ is a feasible point of \eqref{eq:lp_mpc} at the next iteration.
	% This procedure can be repeated indefinitely.
\end{proof}
\begin{remark}
	Since the constraints in \eqref{eq:lp_mpc} are convex inner approximations of the actual constraints any feasible point $\vec{w}$ is also a feasible point of the original \ac{ocp}.
\end{remark}
\acrodef{fr}[CFR]{convex free region}
\acrodef{dp}[DP]{dynamic programming}
\acrodef{ilqr}[iLQR]{iterative linear quadratic regulator}
\acrodef{ip}[IP]{interior point}
\acrodef{lp}[LP]{linear program}
\acrodef{mpc}[MPC]{model predictive control}
\acrodef{nmpc}[NMPC]{nonlinear model predictive control}
\acrodef{nlp}[NLP]{nonlinear program}
\acrodef{ocp}[OCP]{optimal control problem}
\acrodef{qcqp}[QCQP]{quadratically constrained quadratic program}
\acrodef{rhc}[RHC]{receding horizon control}
\acrodef{rk4}[RK4]{Runge-Kutta method of 4\textsuperscript{th} order}
\acrodef{ros}[ROS]{Robot Operating System}
\acrodef{rti}[RTI]{real-time iteration}
\acrodef{sam}[SAM]{sequential approximate method}
\acrodef{ciao}[CIAO]{Convex Inner ApprOximation}
\acrodef{scp}[SCP]{sequential convex programming}
\acrodef{sdp}[SDP]{semi-definite program}
\acrodef{sdf}[SDF]{signed distance function}
\acrodef{slp}[SLP]{sequential linear programming}
\acrodef{sqp}[SQP]{sequential quadratic programming}
\acrodef{socp}[SOCP]{second order cone program}

\end{document}